\documentclass[letterpaper, 10 pt, conference]{ieeeconf}  
\usepackage[colorlinks]{hyperref}
\pdfoutput=1
\usepackage{times}
\usepackage{graphicx}
\usepackage{amsmath}
\usepackage{amssymb}

\newcommand{\bigO}[1]{\mathcal{O}(#1)}

\newcommand{\mc}[1]{\mathcal{#1}}

\newcommand{\vts}[1]{\lvert #1 \rvert}
\newcommand{\Vts}[1]{\lVert #1 \rVert}
\newcommand{\bb}[1]{\mathbb{#1}}
\newcommand\Tstrut{\rule{0pt}{2.6ex}}         
\newcommand\Bstrut{\rule[-1.3ex]{0pt}{0pt}}   

\makeatletter
\newcommand\footnoteref[1]{\protected@xdef\@thefnmark{\ref{#1}}\@footnotemark}
\makeatother

\newcommand{\shorteq}{%
  \settowidth{\@tempdima}{-}
  \resizebox{\@tempdima}{\height}{=}%
}

\let\labelindent\relax
\usepackage[table,x11names]{xcolor}
\usepackage{amsmath}
\usepackage{subcaption}
\usepackage[utf8]{inputenc}
\usepackage[linewidth=1pt]{mdframed}
\usepackage[T1,OT1]{fontenc}
\usepackage[english]{babel}
\usepackage{xcolor}
\usepackage{algorithmicx}
\usepackage[Algorithm,ruled]{algorithm}
\usepackage[font=small,belowskip=-1.5pt]{caption} 
\usepackage[noend]{algpseudocode}
\usepackage{array}
\usepackage{graphicx}
\usepackage{amsfonts}
\usepackage{enumitem}
\usepackage{soul}
\usepackage{hhline}
\usepackage{multirow, makecell}
\usepackage{float}
\usepackage{booktabs}

\usepackage{amsthm}
\usepackage{color}
\usepackage{transparent}
\usepackage{url}
\usepackage[symbol]{footmisc}
\usepackage{setspace}
\usepackage{mathtools}
\theoremstyle{plain}
\newtheorem{theorem}{Theorem}[section]

\newtheorem{definition}{Definition}[section]
\newtheorem{problem}{Problem}[section]

\definecolor{sg}{HTML}{00ff7f}
\definecolor{lb}{HTML}{b0f5ef}
\definecolor{lg}{HTML}{9bfaa8}

\IEEEoverridecommandlockouts                              
                                                          
\begin{document}

\title{Forecasting Trajectory and Behavior of Road-Agents Using Spectral Clustering in Graph-LSTMs }

\author{Rohan Chandra$^{1}$, Tianrui Guan$^{1}$, Srujan Panuganti$^{1}$, Trisha Mittal$^{1}$,\\
Uttaran Bhattacharya$^{1}$, Aniket Bera$^{1}$, Dinesh Manocha$^{1,2}$\\
{\small University of Maryland, College Park}\\
{\small \url{https://gamma.umd.edu/spectralcows}}
}

\maketitle
\begin{abstract}
We present a novel approach for traffic forecasting in urban traffic scenarios using a combination of spectral graph analysis and deep learning. We predict both the low-level information (future trajectories) as well as the high-level information (road-agent behavior) from the extracted trajectory of each road-agent. Our formulation represents the proximity between the road agents using a weighted dynamic geometric graph (DGG). We use a two-stream graph-LSTM network to perform traffic forecasting using these weighted DGGs. The first stream predicts the spatial coordinates of road-agents, while the second stream predicts whether a road-agent is going to exhibit overspeeding, underspeeding, or neutral behavior by modeling spatial interactions between road-agents. Additionally, we propose a new regularization algorithm based on spectral clustering to reduce the error margin in long-term prediction (3-5 seconds) and improve the accuracy of the predicted trajectories. Moreover, we prove a theoretical upper bound on the regularized prediction error. We evaluate our approach on the Argoverse, Lyft, Apolloscape, and NGSIM datasets and highlight the benefits over prior trajectory prediction methods. In practice, our approach reduces the average prediction error by approximately $75$\% over prior algorithms and achieves a weighted average accuracy of $91.2$\% for behavior prediction. Additionally, our spectral regularization improves long-term prediction by up to $70\%$.
\end{abstract}
\section{Introduction}

Autonomous driving is an active area of research and includes many issues related to navigation~\cite{ernest_nav}, trajectory prediction~\cite{traphic}, and behavior understanding~\cite{chandra2019graphrqi, pnas}.
Trajectory prediction is the problem of predicting the short-term (1-3 seconds) and long-term (3-5 seconds) spatial coordinates of various road-agents such as cars, buses, pedestrians, rickshaws, and even animals, etc. Accurate trajectory prediction is crucial for safe navigation. Furthermore, road-agents have different dynamic behaviors that may correspond to aggressive or conservative driving styles~\cite{humanfactor1,humanfactor2,humanfactor3}. While humans can very quickly predict different road-agent behaviors commonly observed in traffic, current autonomous vehicles (AVs) are unable to perform efficient navigation in dense and heterogeneous traffic due to their inability to recognize road-agent behaviors.
\begin{figure}[t]
    \centering
    \includegraphics[width = \columnwidth]{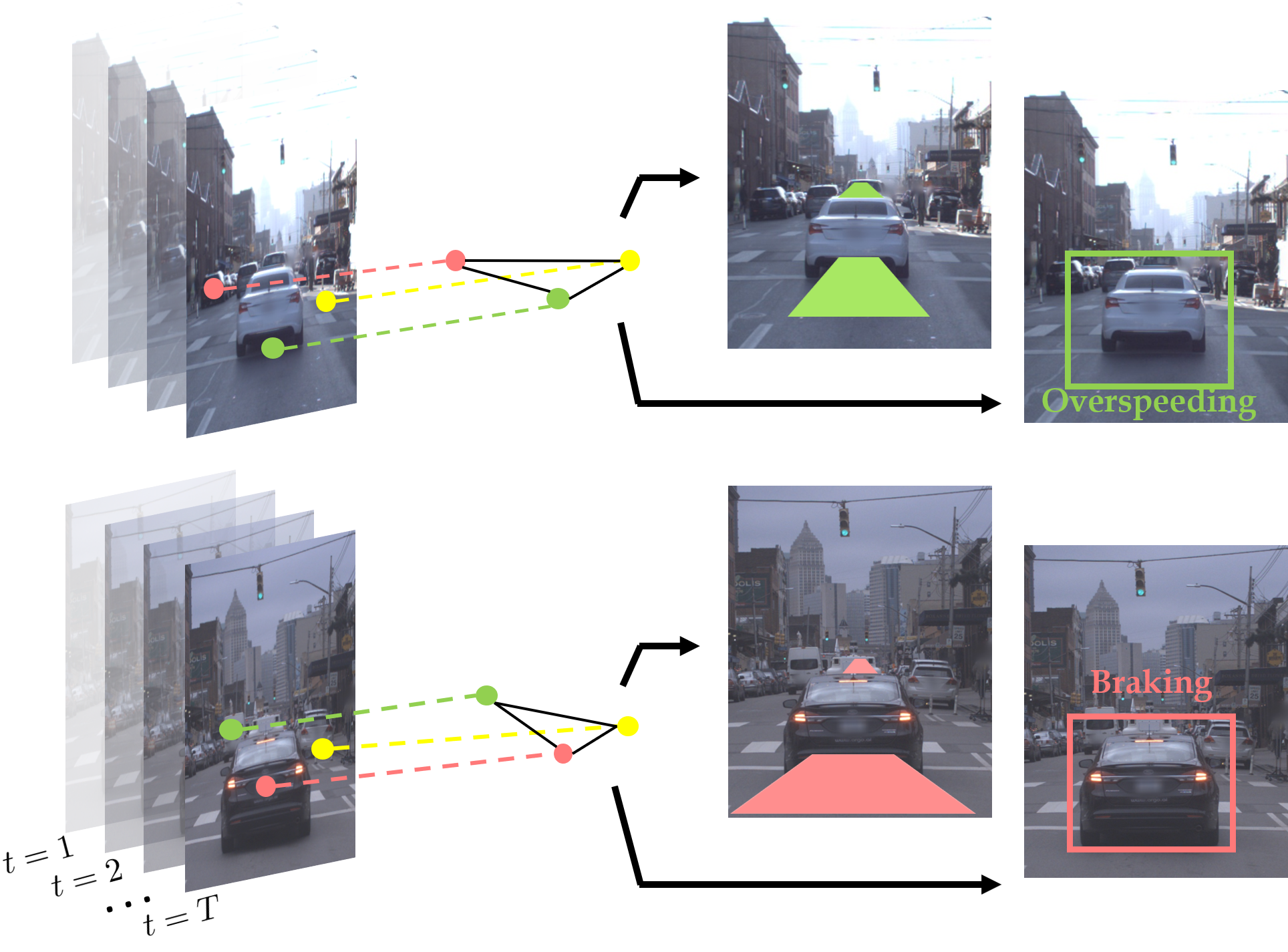}
    \caption{\textbf{Trajectory and Behavior Prediction:} We predict the long-term (3-5 seconds) trajectories of road-agents, as well as their behavior (e.g. overspeeding, underspeeding, etc.), in urban traffic scenes. Our approach represents the spatial coordinates of road-agents (colored points in the image) as vertices of a DGG to improve long-term prediction using a new regularization method. }
    \vspace{-10pt}
    \label{fig:cover}
\end{figure}


While there has been extensive progress in trajectory prediction \mbox{\cite{traphic,nachiket,chandra2019robusttp}}, there has been significantly less research in behavior prediction. The advantage of knowing if a neighboring road-agent is going to overtake another agent or if a road-agent in front is going to brake suddenly is useful for safe navigation. Furthermore, behavior prediction is crucial for making autonomous vehicles socially aware, as opposed to their inherent conservative behavior~\mbox{\cite{brown2017trouble, schwarting2019social, sunberg2017value}} that poses new risks in terms of low efficiency and uncomfortable traveling experiences~\mbox{\cite{seth19traffic}}.

Furthermore, a major challenge in traffic forecasting is ensuring accurate long-term prediction (3-5 seconds). As the prediction horizon increases, the temporal correlations in the data between current and previous time-steps grow weaker, which increases the error-margin of long-term prediction (\cite{anima}, cf. Figure 4 in~\cite{nachiket,traphic}, Figure 3 in~\cite{li2019grip}). 
Some approaches have been developed to reduce the long-term error-margin for trajectory forecasting~\cite{anima}, but they assume knowledge of high-order, non-linear traffic dynamics.

\paragraph{Main Contributions} We present an algorithm for traffic forecasting that disjointedly predicts trajectories as well as road-agent behavior using two separate streams. We represent the inter-road-agent interactions in the traffic using weighted dynamic geometric graphs (DGGs)~\cite{waxman1988routing}, where the vertices represent the road-agents, and the weighted edges are a function of the proximity between the road-agents. Our approach makes no assumptions about the size and shape of the road-agents. Our main contributions include: 

\begin{enumerate}[noitemsep]
    \item A two-stream graph-LSTM network for traffic forecasting in urban traffic.
    The first stream is a conventional LSTM encoder-decoder network that does not account for neighbor vehicles. It is used to predict the spatial coordinates of the future trajectory. We propose a second stream that predicts the eigenvectors of future DGGs, which serve the dual purpose of behavior prediction as well as regularizing the first stream.  
    
    \item To reduce the error of long-term predictions, we propose a new regularization algorithm for sequence prediction models called spectral cluster regularization. 
    \item We derive a theoretical upper bound on the prediction error of the regularized forecasting algorithm in the order of $\bigO{\sqrt{N} \delta_{max}}$, where $N$ is the number of road-agents and $\delta_{max}$ value corresponds to the distance between the two closest road-agents.
    
    \item We present a rule-based behavior prediction algorithm to forecast whether a road-agent is overspeeding (aggressive), underspeeding (conservative), or neutral, based on the traffic behavior classification in psychology literature~\cite{ernestref2, rohanref5}. 
\end{enumerate}
We evaluate our approach on four large-scale urban driving datasets -- NGSIM, Argoverse, Lyft, and Apolloscape. We also perform an exhaustive comparison with the SOTA trajectory prediction methods and report an average RMSE (root mean square error) reduction of approximately 75\% with respect to the next best method. We also achieved a weighted average accuracy of 91.2\% for behavior prediction. Our regularization algorithm improves long-term prediction by up to 70\%. 

\section{Related Work}

Here, we discuss prior work in trajectory prediction, road-agent behavior prediction, and traffic forecasting.

\subsection{Trajectory Prediction}

Trajectory prediction is a well-known problem in statistics~\cite{box2015time}, signal processing~\cite{particle}, and controls and systems engineering~\cite{ljung2001system}. These approaches, however, rely on the knowledge of certain parameters that may not be readily available in traffic videos. In such instances, data-driven methods such as deep learning have become the SOTA for designing trajectory prediction algorithms.

There is some research on trajectory prediction. Deo et al.~\cite{nachiket} combined LSTMs with Convolutional Neural Networks (CNNs) to predict the trajectories of vehicles on sparse U.S. highways. Chandra et al.~\cite{traphic, chandra2019robusttp} proposed algorithms to predict trajectories in urban traffic with high density and heterogeneity. For traffic scenarios with moderate density and heterogeneity, Ma et al.~\cite{ma2018trafficpredict} proposed a method based on reciprocal velocity obstacles. Some additional deep learning-based trajectory prediction methods include~\cite{cnnpredict1,cnnpredict2}. However, these methods only capture road-agent interactions inside a local grid, whereas graph-based approaches such as GRIP~\cite{li2019grip} for trajectory prediction of road-agents and~\cite{cui2018traffic,guo2019attention,Yu2018SpatioTemporalGC,diao2019dynamic} for traffic density prediction consider all interactions independent of local neighborhood restrictions. Our graph representation differs from that of GRIP by storing the graphs of previous time-steps (\ref{subsec: spectral graph theory}). Using our representations, we propose a novel behavior prediction algorithm (\ref{subsec: behavior_protocol}). Additionally, unlike other trajectory prediction methods in the literature, we propose a new Spectral Regularization-based loss function (\ref{sec: spectral_clustering}) that automatically corrects and reduces long-term errors. This is a novel improvement over \textit{all} prior prediction methods that do not handle long-term errors.

\begin{figure*}[t]
    \centering
    \includegraphics[width=.9\textwidth]{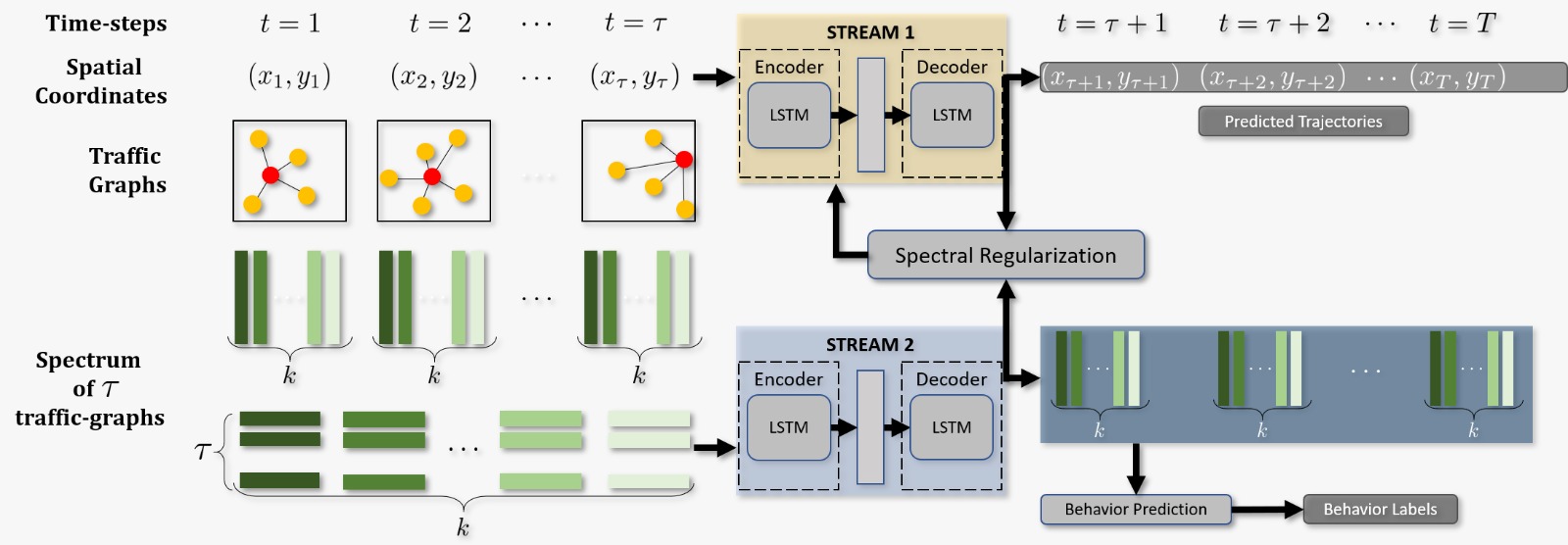}
    \caption{\textbf{Network Architecture}: We show the trajectory and behavior prediction for the $i^{\textrm{th}}$ road-agent (red circle in the DGGs). The input consists of the spatial coordinates over the past $\tau$ seconds as well as the eigenvectors (green rectangles, each shade of green represents the index of the eigenvectors) of the DGGs corresponding to the first $\tau$ DGGs. We perform spectral clustering on the predicted eigenvectors from the second stream to regularize the original loss function and perform back-propagation on the new loss function to improve long-term prediction.}
    \label{fig:network}
    \vspace{-10pt}
\end{figure*}




\subsection{Road-Agent Behavior Prediction}

Current autonomous vehicles lack social awareness due to their inherent conservative behavior~\mbox{\cite{brown2017trouble, schwarting2019social, sunberg2017value}}. Overly conservative behavior present new risks in terms of low efficiency and uncomfortable traveling experiences~\mbox{\cite{seth19traffic}}. Real-world examples of problems caused by AVs that are not socially adaptable can be seen in this video\footnote{\url{https://www.youtube.com/watch?v=Rm8aPR0aMDE}}. The notion of using driver behavior prediction to make the AVs socially aware is receiving attention~\mbox{\cite{schwarting2019social}}.

Current driving behavior modeling methods are limited to traffic psychology studies where predictions for driving behavior are made offline, based on either driver responses to questionnaires or data collected over a period of time. Such approaches are not suitable for online behavior prediction. In contrast, our behavior prediction algorithm is the first computationally online approach that does not depend on offline data and manually tunable parameters. In the remainder of this section, we review some of the prior behavior modeling approaches and conclude by pointing out the advantages of our approach.

Many studies have been performed behavior modeling by identifying factors that contribute to different driver behaviors classes such as aggressive, conservative, or moderate driving. These factors can be broadly categorized into four categories. The first category of factors that indicate road-agent behavior is driver-related. These include characteristics of drivers such as age, gender, blood pressure, personality, occupation, hearing, and so on~\cite{ernestref2, big5, ernestref9}. 
The second category corresponds to environmental factors such as weather or traffic conditions~\cite{behaviorref-category2-1,behaviorref-category2-2}.
The third category refers to psychological aspects that affect driving styles. These could include drunk driving, driving under the influence, state of fatigue, and so on~\cite{behaviorref-category3-2,behaviorref-category3-1}.
The final category of factors contributing to driving behavior are vehicular factors such as positions, acceleration, speed, throttle responses, steering wheel measurements, lane changes, and brake pressure~\cite{ernestref8,ernestref12, ernestref13, ernest, chandra2019graphrqi}.

A recent data-driven behavior prediction approach~\cite{chandra2019graphrqi} also models traffic through graphs. The method predicts the driving behavior by training a neural network on the eigenvectors of the DGGs using supervised machine learning. Apart from behavior modeling, several methods have used machine learning to predict the intent of road-agents~\cite{zyner2018recurrent,zyner2017long}. The proposed behavior prediction algorithm in this paper extends the approach in~\cite{chandra2019graphrqi} by predicting sequences of eigenvectors for future time-steps. Compared to these prior methods, are algorithm is online, computationally tractable and does not depend on any other information other than the vehicle coordinates.

\section{Background and Overview}
In this section, we define the problem statement and give a brief overview of spectral Dynamic Geometric Graphs (DGGs) in the context of road-agents. 

\subsection{Problem Statement}

We first present a definition of a vehicle trajectory:

\begin{definition}
\textbf{Trajectory:} The trajectory for the $i^\textrm{th}$ road agent is defined as a sequence $\Psi_i(a,b) \in \{ \mathbb{R}^2 \}$, where $\Psi_i(a,b) = \left \{ [x_t, y_t]^\top | \ t \in [a, b] \right\}$. $[x,y] \in \bb{R}^2$ denotes the spatial coordinates of the road-agent in meters according to the world coordinate frame and $t$ denotes the time instance.
\end{definition}

We define \textit{traffic forecasting} as solving the following two problem statements, simultaneously, but separately using two separate streams.

\begin{problem}
\textbf{Trajectory Prediction: }In a traffic video with $N$ road agents, given the trajectory $\Psi_i(0, \tau)$, predict $\Psi_i (\tau^+, T)$ for each road-agent $v_i, i\in [0,N]$.
\label{prob: 1}
\end{problem}

\begin{problem}
\textbf{Behavior Prediction:} In a traffic video with $N$ road agents, given the trajectory, $\Psi_i(0, \tau)$, predict a label from the following set, \{ Overspeeding, Neutral, Underspeeding\} for each road-agent $v_i, i\in [0,N]$.
\label{prob: 2}
\end{problem}

\subsection{Weighted Dynamic Geometric Graphs (DGGs)}
\label{subsec: spectral graph theory}

We assume that the trajectories of all the vehicles in the video are provided to us as the input. Given this input, we first construct a DGG~\cite{waxman1988routing} at each time-step. In a DGG, the vehicles represent the vertices and the edge weights are a function of the euclidean distance between the vertices. This function~\cite{belkin2003laplacian} is given by,

\begin{equation}
    f(v_i,v_j) = e^{-d(v_i,v_j)}
    \label{eq: kernel}
\end{equation}

\noindent where $v_i$ and $v_j$ are the $i^{\textrm{th}}$ and $j^{\textrm{th}}$ vertices and $d$ is the euclidean distance function.

We represent traffic at each time instance using a DGG $\mc{G}$ of size $N \times N$, with the spatial coordinates of the road-agent representing the set of vertices $\mc{V} = \{ v_1, v_2, \dots, v_n \}$ and a set of undirected, weighted edges, $\mc{E}$. Two road-agents are said to be connected through an edge if $d(v_i,v_j) < \mu$, where $d(v_i,v_j)$ represents the Euclidean distance between the road-agents and $\mu$ is a heuristically chosen threshold parameter. In our experiments, we choose $\mu=10$ meters, taking into account the typical size of road-agents and the width of the road. 

For a DGG, $\mc{G}$, we define the symmetrical adjacency matrix, $A \in \mathbb{R}^{N \times N}$ as,

\begin{equation}
\resizebox{0.7\columnwidth}{!}{
$A(i,j)=
     \begin{cases}
       e^{-d(v_i,v_j)}  & \text{if $d(v_i,v_j) < \mu,i \neq j$ },\\
      0 &\text{otherwise.}
     \end{cases}$
     }
     \label{eq: similarity_function}
\end{equation}

\noindent Equation~\ref{eq: kernel} denotes the interactions between any two road-agents, $v_i$ and $v_j$. This implies that road-agents at a greater distance are assigned a lower weight, while road-agents in close proximity are assigned a higher weight. This follows the intuition that each road-agent needs to pay more attention to nearby agents than those farther away to avoid collisions.

For the adjacency matrix $A$ at each time instance, the corresponding degree matrix $D \in \mathbb{R}^{N \times N}$ is a diagonal matrix with main diagonal $D(i, i) = \sum_{j=1}^N A(i, j)$ and 0 otherwise. The unnormalized Laplacian matrix $L  = D  - A$ of the graph is defined as the symmetric matrix,

\begin{equation}
\resizebox{0.6\columnwidth}{!}{
$L(i,j) =
     \begin{cases}
       D(i,i)  &\text{if $i=j$},\\
      -e^{-d(v_i,v_j)} &\text{if $d(v_i,v_j) < \mu$}, \\
      0 &\text{otherwise.}
     \end{cases}$
     }
\end{equation}

The Laplacian matrix for each time-step is correlated with the Laplacian matrices for all previous time-steps. Let the Laplacian matrix at a time instance $t$ be denoted as $L_t$. Then, the laplacian matrix for the next time-step, $L_{t+1}$ is given by the following update,

\begin{equation}
\resizebox{0.4\columnwidth}{!}{
$L_{t+1} =
\left[
\begin{array}{c|c}
L_{t} \Bstrut & 0 \Bstrut\\
\hline
0 \Tstrut & 1
\end{array}
\right] + \delta\delta^\top$,
}
\label{eq: A_update}
\end{equation}

\noindent where $\delta\delta^\top$ is a perturbation matrix represented by an outer product of rank 2. Here, $\delta \in \bb{R}^{(N+1) \times 2} $ is a sparse matrix $\Vts{\delta}_0 \ll N$, where N represents the total number of road-agents at time-step $t$. The presence of a non-zero entry in the $j^\textrm{th}$ row of $\delta$ implies that the $j^\textrm{th}$ road-agent has observed a new neighbor, that has now been added to the current DGG. During training time, the size of $L_t$ is fixed for all time $t$ and is initialized as a zero matrix of size $N$x$N$, where $N$ is max number of agents (different $N$ is used for different datasets). For instance, $N=270$ is used for Lyft Level 5 dataset. At current time $t$, if the $N<270$, the zeros in $L_t$ will simply be updated with new values. Once $N=270$, $L_t$ is reset to zero and the process repeats. During test time, trained models for stream 1 predict trajectories based only on past trajectories; these models for stream 1 do not use graphs. Trained model for stream 2, however, generate traffic-graphs in realtime for behavior prediction at test time.
The matrix $U \in \mathbb{R}^{n \times k} := \{ u_j \in \mathbb{R}^{n} | j = 1 \dots k\}$ of eigenvectors of $L$ is called the \emph{spectrum} of $L$, and can be efficiently computed using eigenvalue algorithms.



\section{Trajectory and Behavior Forecasting}
The overall flow of the approach is as follows:

\begin{enumerate}[noitemsep]
    \item Our input consists of the spatial coordinates over the past $\tau$ seconds as well as the eigenvectors of the DGGs corresponding to the first $\tau$ DGGs.
    
    \item \textit{Solving Problem~\ref{prob: 1}: }The first stream accepts the spatial coordinates and uses an LSTM-based sequence model~\cite{graves2013generating} to predict $\Psi_i (\tau^{+}, T)$ for each $v_i, i \in [0,N]$, where $\tau^{+} = \tau + 1$.
    
    \item \textit{Solving Problem~\ref{prob: 2}:} The second stream accepts the eigenvectors of the input DGGs and predicts the eigenvectors corresponding to the DGGs for the next $\tau$ seconds. The predicted eigenvectors form the input to the behavior prediction algorithm in Section~\ref{subsec: behavior_protocol} to assign a behavior label to the road-agent.
    
    \item Stream 2 is used to regularize stream 1 using a new regularization algorithm presented in Section~\ref{sec: spectral_clustering}. We derive the upper bound on the prediction error of the regularized forecasting algorithm in Section~\ref{sec: upper_bound}.
\end{enumerate}

\subsection{Network Overview}

We present an overview of our approach in Figure~\ref{fig:network} and defer the technical implementation details of our network to the supplementary material. Our approach consists of two parallel LSTM networks (or streams) that operate separately.

\textbf{Stream 1:} The first stream is an LSTM-based encoder-decoder network~\cite{graves2013generating} (yellow layer in Figure~\ref{fig:network}). The input consists of the trajectory history, $\Psi_i(0, \tau)$ and output consists of $\Psi_i (\tau^{+}, T)$ for each road-agent $v_i, i\in [0,N]$. 

\textbf{Stream 2:} The second stream is also an LSTM-based encoder-decoder network (blue layer in Figure~\ref{fig:network}). To prepare the input to this stream, we first form a sequence of DGGs, $\left\{\mc{G}_t | \ t\in [0,\tau] \right\}$ for each time instance of traffic until time $\tau$. For each DGG, $\mc{G}_t$, we first compute its corresponding Laplacian matrix, $L_t$ and use SOTA eigenvalue algorithms to obtain the spectrum, $U_t$ consisting of the top $k$ eigenvectors of length $n$. We form $k$ different sequences, $\left\{ \mc{S}_j | \ j\in [0,k] \right \}$, where each $\mc{S}_j = \{u_j\}$ is the set containing the $j^{\textrm{th}}$ eigenvector from each $U_t$ corresponding to the $t^{\textrm{th}}$ time-step, with $\vts{\mc{S}_j} = \tau$. 

The second stream then accepts a sequence, $\mc{S}_j$, as input to predict the $j^{\textrm{th}}$ eigenvectors for the next $T - \tau$ seconds. This is repeated for each $\mc{S}_j$. The resulting sequence of spectrums, $\left\{ \mc{U}_t | \ t\in [\tau^{+} , T] \right \}$ are used to reconstruct the sequence, $\left \{ \mc{L}_t | \ t\in [\tau^{+} , T] \right \}$, which is then used to assign a behavior label to a road-agent, as explained below.

\subsection{Trajectory Prediction}

The first stream is used to solve Problem~\ref{prob: 1}. We clarify at this point that stream 1 does not take into account road-agent interactions. We use spectral clustering (discussed later in Section~\ref{sec: spectral_clustering}) to model these interactions. It is important to further clarify that the trajectories predicted from stream 1 are not affected by the behavior prediction algorithm (explained in the next Section).

\subsection{Behavior Prediction Algorithm}
\label{subsec: behavior_protocol}
We define a rule-based behavior algorithm (blue block in Figure~\ref{fig:network}) to solve Problem~\ref{prob: 2}. This is largely due to the fact that most data-driven behavior prediction approaches require large, well-annotated datasets that contain behavior labels. Our algorithm is based on the predicted eigenvectors of the DGGs of the next $\tau$ seconds.

The degree of $i^{\textrm{th}}$ road-agent, ($\theta_i \leq n$), can be computed from the diagonal elements of the Laplacian matrix $L_t$. $\theta_i$ measures the total number of distinct neighbors with which road-agent $v_i$ has shared an edge connection until time $t$. As $L_t$ is formed by simply adding a row and column to $L_{t-1}$, the degree of each road-agent monotonically increases. Let the rate of increase of $\theta_i$ be denoted as $\theta_i^{'}$. Intuitively, an aggressively overspeeding vehicle will observe new neighbors at a faster rate as compared to a road-agent driving at a uniform speed. Conversely, a conservative road-agent that is often underspeeding at unconventional spots such as green light intersections (Figure~\ref{fig:cover}) will observe new neighbors very slowly. This intuition is formalized by noting the change in $\theta_i$ across time-steps. In order to make sure that slower vehicles (conservative) did not mistakenly mark faster vehicles as new agents, we set a condition where an observed vehicle is marked as `new' if and only if the speed of the observed vehicle is less than the active vehicle (or ego-vehicle). To predict the behavior of the $i^{\textrm{th}}$ road-agent, we follow the following steps:

\begin{enumerate}[noitemsep]
    \item Form the set of predicted spectrums from stream 2,  $\left\{ \mc{U}_t | \ t\in [\tau^{+} , T] \right \}$. We compute the eigenvalue matrix, $\Lambda$, of $L_t$ by applying theorem 5.6 of~\cite{demmel1997applied} to $L_{t-1}$. We explain the exact procedure in the supplemental version.
    \item For each $U_t \in \mc{U}$, compute $L_t = U_t\Lambda U_t^\top$.
    \item $\theta_i = i^{\textrm{th}}$ element of diag($L_t$), where ``diag'' is the diagonal matrix operator.
    \item $\theta_i^{'} = \frac{\Delta \theta_i}{\Delta t}$.
\end{enumerate}

\noindent where $\Lambda$ is the eigenvalue matrix of $L_t$. Based on heuristically pre-determined threshold parameters $\lambda_1$ and $\lambda_2$, we define the following rules
to assign the final behavior label: Overspeeding ($\theta^{'} > \lambda_1$), Neutral ($\lambda_2 \leq \theta^{'} \leq \lambda_1$), and Underspeeding ($\theta^{'} < \lambda_2$). 

Note that since human behavior does not change instantly at each time-step, our approach predicts the behavior over time periods spanning several frames.
\subsection{Spectral Clustering Regularization}
\label{sec: spectral_clustering}

The original loss function of stream 1 for the $i^{\textrm{th}}$ road-agent in an LSTM network is given by,
\begin{equation}
\resizebox{0.65\columnwidth}{!}{
    $F_i = -\sum_{t=1}^T \log Pr(x_{t+1} | \mu_t, \sigma_t, \rho_t)$
    }
    \label{eq: original_loss}
\end{equation}

\noindent Our goal is to optimize the parameters, $\mu^{*}_t, \sigma^{*}_t$, that minimize equation~\ref{eq: original_loss}. Then, the next spatial coordinate is sampled from a search space defined by $\mathcal{N}(\mu^{*}_t, \sigma^{*}_t)$. The resulting optimization forces $\mu_t, \sigma_t$ to stay close to the next spatial coordinate. However, in general trajectory prediction models, the predicted trajectory diverges gradually from the ground-truth, causing the error-margin to monotonically increase as the length of the prediction horizon increases (\cite{anima}, cf. Figure 4 in~\cite{nachiket,traphic}, Figure 3 in~\cite{li2019grip}). The reason for this may be that while equation~\ref{eq: original_loss} ensures that $\mu_t, \sigma_t$ stays close to the next spatial coordinate, it does not, however, guarantee the same for $\hat x_{t+1} \sim \mathcal{N}(\mu_t, \sigma_t)$. Our solution to this problem involves regularizing equation~\ref{eq: original_loss} by adding appropriate constraints on the parameters, $\mu_t, \sigma_t$, such that sampled coordinates from $\mathcal{N}(\mu^{*}_t, \sigma^{*}_t)$ are close to the ground-truth trajectory.

We assume the ground-truth trajectory of a road-agent to be equivalent to their ``preferred" trajectory, which is defined as the trajectory a road-agent would have taken in the absence of other dynamic road-agents. Preferred trajectories can be obtained by minimizing the Dirichlet energy of the DGG, which in turn can be achieved through spectral clustering on the road-agents~\cite{dirichlet}. Our regularization algorithm (shown in the yellow arrow in Figure~\ref{fig:network}) is summarized below. For each road-agent, $v_i$: 

\begin{enumerate}[noitemsep]
    \item The second stream computes the spectrum sequence, $\{U_{T+1}, \ldots, U_{T+\tau}\}$.
    \item For each $U$, perform spectral clustering~\cite{von2007tutorial} on the eigenvector corresponding to the second smallest eigenvalue.
    \item Compute cluster centers from the clusters obtained in the previous step.
    \item Identify the cluster to which $v_i$ belongs and retrieve the cluster center, $\mu_c$ and deviation, $\sigma_c$. 
\end{enumerate}
Then for each road-agent, $v_i$, the regularized loss function, $F^{\textrm{reg}}_i$, for stream 1 is given by,

\begin{equation}
\resizebox{.9\columnwidth}{!}{
    $\sum_{t=1}^T \Big (- \log Pr(\hat y_{t+1} | \mu_t, \sigma_t, \rho_t \Big ) + b_1\Vts{\mu_t - \mu_c}_2 + b_2\Vts{\sigma_t - \sigma_c}_2$
    }
    \label{eq: oregularized_loss}
\end{equation}

\noindent where $b_1 = b_2=0.5$ are regularization constants. The regularized loss function is used to backpropagate the weights corresponding to $\mu_t$ in stream 1. Note that $F^{\textrm{reg}}_i$ resembles a Gaussian kernel. This makes sense as the Gaussian kernel models the Euclidean distance non-linearly -- greater the Euclidean distance, smaller the Gaussian kernel value and vice versa. This behavior is similarly captured by Equation~\ref{eq: kernel}). Furthermore, we can use Equation~\ref{eq: oregularized_loss} to predict multiple modes\cite{nachiket} by computing maneuver probabilities using $\mu, \sigma$ following the approach in Section 4.3 of~\cite{nachiket}.

\section{Upper Bound for Prediction Error}
\label{sec: upper_bound}
In this section, we derive an upper bound on the prediction error, $\phi_j$, of the first stream as a consequence of spectral regularization. We present our main result as follows,

\begin{theorem}
    $\phi_j \leq \frac{  \Vts{\delta_t\delta_t^\top}_2  }{  min(\lambda_j, \Lambda)  }$, where $min(\lambda_j, \Lambda)$ denotes the minimum distance between $\lambda_j$ and $\lambda_k \in \Lambda \setminus {\lambda_j}$.
    \label{eq: theoretical_bound}
\end{theorem}

\begin{proof}

At time instance $t$, the Laplacian matrix, $L_t$, its block form,\resizebox{0.18\columnwidth}{!}{
$
\left[ 
\begin{array}{c|c}
L_{t} \Bstrut & 0 \Bstrut\\
\hline
0 \Tstrut & 1
\end{array}
    \right ]$}, denoted as block($L_t$), and the laplacian matrix for the next time-step, $L_{t+1}$ are described by Equation~\ref{eq: A_update}. We compute the eigenvalue matrix, $\Lambda$, of $L_t$ by applying theorem 5.6 of~\cite{demmel1997applied} to $L_{t-1}$. 
    
    LSTMs make accurate sequence predictions if elements of the sequence are correlated across time, as opposed to being generated randomly. In a general sequence of eigenvectors, the eigenvectors may not be correlated across time. Consequently, it is difficult for LSTM networks to predict the sequence of eigenvectors, $\mc{U}$ accurately. This may adversely affect the behavior prediction algorithm described in Section~\ref{subsec: behavior_protocol}. Our goal is now to show there exist a correlation between Laplacian matrices across time-steps and that this correlation is lower-bounded, that is, there exist sufficient correlation for accurate sequence modeling of eigenvectors.

    Proving a lower-bound for the correlation is equivalent to proving an upper-bound for the noise, or error distance, between the $j^{\textrm{th}}$ eigenvectors of $L_t$ and $L_{t+1}$. We denote this error distance through the angle $\phi_j$. From Theorem 5.4 of~\cite{demmel1997applied}, the numerator of bound corresponds to the frobenius norm of the error between $L_t$ and $L_{t+1}$. In our case, the update to the Laplacian matrix is given by Equation~\ref{eq: A_update} where the error matrix is $\delta \delta^\top$.
\end{proof}

\noindent In Theorem~\ref{eq: theoretical_bound}, $\phi_j \ll 1$ and $\delta$ is defined in equation~\ref{eq: A_update}. $\lambda_j$ represents the $j^{\textrm{th}}$ eigenvalue and $\Lambda$ represents all the eigenvalues of $L_t$. If the maximum component of $\delta_t$ is $\delta_{max}$, then $\phi_j = \bigO{\sqrt{N} \delta_{max}}$.
Theorem~\ref{eq: theoretical_bound} shows that in a sequence of $j^{\textrm{th}}$ eigenvectors, the maximum angular difference between successive eigenvectors is bounded by $\bigO{\sqrt{N} \delta_{max}}$. By setting $N=270$ (number of road-agents in Lyft), and $\delta_{max} \coloneqq e^{-3} = 0.049$ (width of a lane), we observe a theoretical upper bound of $0.8$ meters. A smaller value of $\phi_j$ indicates a greater similarity between successive eigenvectors, thereby implying a greater correlation in the sequence of eigenvectors. This allows sequence prediction models to learn future eigenvectors efficiently.

An alternative approach to computing the spectrums $\{U_{T+1}, \ldots, U_{T+\tau}\}$ is to first form traffic-graphs from the predicted trajectory given as the output from the stream 1. After obtaining the corresponding Laplacian matrices for these traffic-graphs, standard eigenvalue algorithms can be used to compute the spectrum sequence. This is, however, a relatively sub-optimal approach as in this case, $\phi = \bigO{N L_{max}}$, with $L_{max} \gg \delta_{max}$.

\begin{table*}
\centering
\caption{\textbf{Main Results:} We report the Average Displacement Error (ADE) and Final Displacement Error (FDE) for prior road-agent trajectory prediction methods in meters (m). Lower scores are better and \textbf{bold} indicates the SOTA. We used the original implementation and results for GRIP~\mbox{\cite{li2019grip}} and Social-GAN~\cite{social-gan}. `-' indicates that results for that particular dataset are not available. \textbf{Conclusion:} Our spectrally regularized method (``S1 + S2'') outperforms the next best method (GRIP) by upto $70\%$ as well as the ablated version of our method (``S1 Only'') by upto $75\%$.}
\resizebox{.9\linewidth}{!}{%
\begin{tabular}{lcccccccccccc}
\toprule[1.1pt]
Dataset (Pred. Len.)   &
\multicolumn{8}{c}{Comaprison Methods}
&
\multicolumn{2}{c}{Ablation}
& \multicolumn{2}{c}{\textbf{Our Approach}}            \\
\midrule
                    &
\multicolumn{2}{c}{CS-LSTM }         &
\multicolumn{2}{c}{TraPHic  }         &
\multicolumn{2}{c}{Social-GAN  }         &
\multicolumn{2}{c}{GRIP  }         &
\multicolumn{2}{c}{S1 Only }  &
\multicolumn{2}{c}{S1 + S2 \Tstrut }            \\
&
ADE & FDE &
ADE & FDE &
ADE & FDE &
ADE & FDE &
ADE & FDE &
ADE & FDE \\
\midrule
Lyft (5 sec.)           &   4.423  &  8.640  & 5.031  &  9.882   &  7.860  &  14.340  & -  &  - & 5.77  &  11.20  &  \textbf{2.65} &  \textbf{2.99}\\ 
Argoverse (5 sec.)      &    1.050  &  3.085  & 1.039  &  3.079   &  3.610   &  5.390  & -  &- &  2.40 &  3.09  &  \textbf{0.99} &  \textbf{1.87} \\ 
Apolloscape (3 sec.)    &    2.144  &  11.699  & 1.283  &  11.674  & 3.980  & 6.750  & 1.25 &  2.34  & 2.14  & 9.19  &  \textbf{1.12} &  \textbf{2.05}\\ 
NGSIM (5 sec.)  &   7.250&10.050  & 5.630&9.910  & 5.650&10.290  & 1.61  &  3.16  &  1.31  &  2.98  &\textbf{0.40}  & \textbf{1.08}\\ 
\bottomrule[1.1pt]
\end{tabular}
}
\label{tab: accuracy}
\end{table*}

\section{Experiments and Results}

We begin by listing the datasets used in our approach in Section~\ref{subsec: datasets}. We list the evaluation metrics used and methods compared within Section~\ref{subsec: eval_metrics_methods}. We analyze the main results and discuss the results of comparison methods and ablation studies of our approach in Section~\mbox{\ref{subsec: analysis_and_discussion}}. In Section~\mbox{\ref{subsec: upper_bound_analysis}}, we analyse the theoretical upper bound in the context of long-term prediction. We present an ablation analysis of the radius parameter $\mu$ in Section~\mbox{\ref{subsec: mu_ablation}}.
We make all the implementation and training details available in the supplementary material.

\subsection{Datasets}
\label{subsec: datasets}

We use both sparse (NGSIM~\cite{ngsim}) as well as dense (Lyft Level 5~\cite{lyft2019}, Argoverse Motion Forecasting~\cite{Argoverse}, and the Apolloscape Trajectory~\cite{ma2018trafficpredict}) trajectory prediction datasets for evaluation. We give a brief description of all the datasets in the supplemental version.

\subsection{Evaluation Metrics and Methods}
\label{subsec: eval_metrics_methods}
\subsubsection{Metrics}
For trajectory prediction, we use the standard metrics followed by prior trajectory prediction approaches~\cite{social-lstm,social-gan,traphic,nachiket,chandra2019robusttp}.
\begin{enumerate}[noitemsep]
    \item Average Displacement Error (ADE): The root mean square error (RMSE) of all the predicted positions and real positions during the prediction window.
    \item Final Displacement Error (FDE): The RMSE distance between the final predicted positions at the end of the predicted trajectory and the corresponding true location.
\end{enumerate}

For behavior prediction, we report a weighted classification accuracy (W.A.) over the 3 class labels: \{overspeeding, neutral, underspeeding\}. 

\subsubsection{Methods} 

We compare our approach with SOTA trajectory prediction approaches for road-agents. Our definition of SOTA is not limited to ADE/FDE values. We consider SOTA additionally with respect to the deep learning architecture used in a different approach.
Combined, our basis for selecting SOTA methods not only evaluates the ADE/FDE scores but also evaluates the benefits of using the two-stream network versus other deep learning-based architectures.

\begin{itemize}[noitemsep]
    \item Deo et al.~\cite{nachiket} (CS-LSTM): This method combines CNNs with LSTMs to perform trajectory prediction on U.S. highways.
    \item Chandra et al.~\cite{traphic} (TraPHic): This approach also uses a CNN + LSTM approach along with spatial attention-based pooling to perform trajectory prediction of road-agents in dense and heterogeneous traffic.
    \item Gupta et al.~\cite{social-gan} (Social-GAN): This GAN-based trajectory prediction approach is originally trained on pedestrian crowd datasets. The method uses the encoder-decoder architecture to act as the generator and trains an additional encoder as the discriminator.   
    \item Li et al.~\cite{li2019grip} (GRIP): This is a graph-based trajectory prediction approach that replaces standard CNNs with graph convolutions and combines GCNs with an encoder-decoder framework. 
\end{itemize}

We use the publicly available implementations for CS-LSTM, TraPHic, and Social-GAN, and train the entire model on all three datasets. We performed hyper-parameter tuning on all three methods and reported the best results. Moreover, we compare with the officially published results for GRIP as reported on the NGSIM~\mbox{\cite{li2019grip}} and the Apolloscape datasets only\footnote{\url{http://apolloscape.auto/leader_board.html}}.

\subsection{Analysis and Discussion}
\label{subsec: analysis_and_discussion}

We compare the ADE and FDE scores of our predicted trajectories with prior methods in Table~\ref{tab: accuracy} and show qualitative results in the supplementary material. We compare with several SOTA trajectory prediction methods and reduce the average RMSE by approximately 75\% with respect to the next best method (GRIP). 

\textbf{Ablation Study of Stream 1 (S1 Only) vs. Both Streams (S1 + S2):} To highlight the benefit of the spectral cluster regularization on long-term prediction, we remove the second stream and only train the LSTM encoder-decoder model (Stream 1) with the original loss function (equation~\mbox{\ref{eq: original_loss}}). Our results (Table~\mbox{\ref{tab: accuracy}}, last four columns) show that regularizing stream 1 reduces the FDE by up to 70\%. This is as expected since stream 1 does not take into account neighbor information. Therefore, it should also be noted that stream 1 performs poorly in dense scenarios but rather well in sparse scenarios. This is evident from Table~\mbox{\ref{tab: accuracy}} where stream 1 outperforms comparison methods on the sparse NGSIM dataset with ADE less than $1$m.

Additionally, Figure~\mbox{\ref{fig: rmse}} shows that in the presence of regularization, the RMSE for our spectrally regularized approach (``both streams", purple curve) is much lower than that of stream 1 (red curve) across the entire prediction window.

\textbf{RMSE depends on traffic density: } The upper bound for the increase in RMSE error is a function of the density of the traffic since $\phi = \bigO{\sqrt{N} \delta_{max}}$, where $N$ is the total number of agents in the traffic video and $\delta_{max} = 0.049$ meters for a three-lane wide road system. The NGSIM dataset contains the sparsest traffic with the lowest value for $N$ and therefore the RMSE values are lower for the NGSIM ($0.40/1.08$) compared to the other three datasets that contain dense urban traffic.

\textbf{Comparison with other methods:} Our method learns weight parameters for a spectral regularized LSTM network (Figure~\mbox{\ref{fig:network}}), while GRIP learns parameters for a graph-convolutional network (GCN). We outperform GRIP on the NGSIM and Apolloscape datasets, while comparisons on the remaining two datasets are unavailable. TraPHic and CS-LSTM are similar approaches. Both methods require convolutions in a heuristic local neighborhood. The size of the neighborhood is specifically adjusted to the dataset that each method is trained on. We use the default neighborhood parameters provided in the publicly available implementations, and apply them to the NGSIM, Lyft, Argoverse, and Apolloscape datasets. We outperform both methods on all benchmark datasets. Lastly, Social-GAN is trained on the scale of pedestrian trajectories, which differs significantly from the scale of vehicle trajectories. This is primarily the reason behind Social-GAN placing last among all methods.

\subsection{Long-Term Prediction Analysis}
\label{subsec: upper_bound_analysis}

The goal of improved long-term prediction is to achieve a lower FDE, as observed in our results in Table~\mbox{\ref{tab: accuracy}}. We achieve this goal by successfully upper-bounding the worst-case maximum FDE that can theoretically be obtained. These upper bounds are a consequence of the theoretical results in Section~\mbox{\ref{sec: upper_bound}}. We denote the worst-case theoretical FDE by T-FDE. This measure represents the maximum FDE that can be obtained using Theorem~\mbox{\ref{eq: theoretical_bound}} under fixed assumptions. In Table\mbox{~\ref{tab: upper_bound}}, we compare the T-FDE with the empirical FDE results obtained in Table~\mbox{\ref{tab: accuracy}}. The T-FDE is computed by,

\begin{equation}
    \textrm{T-FDE} = \frac{\phi}{n}\times (T-\tau)
    \label{eq: T-FDE}
\end{equation}

\noindent The formula for T-FDE is derived as follows. The RMSE error incurred by all vehicles at a current time-step during spectral clustering is bounded by $\phi$ (Theorem\mbox{~\ref{eq: theoretical_bound}}). Let $n = \frac{N}{T} = 10$ be the average number of vehicles per frame in each dataset. Then, at a single instance in the prediction window,
the increase in RMSE for a single agent is bounded by $\frac{\phi}{n}$. As $T-\tau$ is the length of the prediction window, the total increase in RMSE over the entire prediction window is given by $\textrm{T-FDE} = \frac{\phi}{n}\times (T-\tau)$. We do not have the data needed to compute $\phi$ for the NGSIM dataset as the total number of lanes are not known.

        

\begin{table}
  \caption{\textbf{Upper Bound Analysis:} $\phi$ is the upper bound on the RMSE for all agents at a time-step. $T-\tau$ is the length of the prediction window. T-FDE~(Eq. \mbox{\ref{eq: T-FDE}}) is the theoretical FDE that should be achieved by using spectral regularization. The FDE results are obtained from Table~\mbox{\ref{tab: accuracy}}. The \% agreement is the agreement between the T-FDE and FDE computed using $\frac{\textrm{T-FDE}}{\textrm{FDE}}$ if T-FDE$<$FDE, else 100\%. \textit{Conclusion:} Theorem~\mbox{\ref{eq: theoretical_bound}} is empirically verified with at least 73\% guarantee.}
  \centering
  \resizebox{\columnwidth}{!}{%
  \begin{tabular}{lccccc}
  \toprule
  Dataset  & $\phi$ &  ($T-\tau$)& T-FDE &  FDE & \% Agreement\\
    \midrule
    Lyft Level 5 & 0.80 &30 & 2.46 & 2.99&82\%\\
    Apolloscape & 1.50 & 10 & 1.50 & 2.05 & 73\%\\
    Argoverse & 0.64 &30  & 1.95 & 1.87 & 100\% \\
    \bottomrule
  \end{tabular}
  \label{tab: upper_bound}
  }
 
  \vspace{-10pt}
\end{table}

We note a $73\%, 82\%, 100\%$ agreement between the theoretical FDE and the empirical FDE on the Apolloscape, Lyft, and Argoverse datasets, respectively. The main cause for disagreements in the first two datasets is the choice for the value of $\delta_\textrm{max} = 0.049$ during the computation of $\phi$. This value is obtained for a three-lane wide road system that was observed in majority of the videos in both datasets. However, it may be the case that several videos contain one- or two-lane traffic. In such cases, the values for $\delta_\textrm{max}$ changes to $0.36$ and $0.13$, respectively, thereby increasing the upper bound for increase in RMSE.

Note, in Figure~\mbox{\ref{fig: rmse}}, the increase in RMSE for our approach (purple curve) is much lower than that of other methods, which is due to the upper bound induced by spectral regularization.


\begin{figure}
    \centering
    \includegraphics[width=\columnwidth]{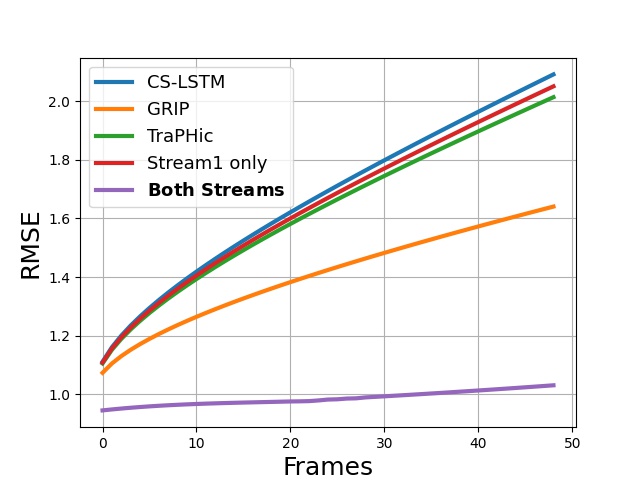}
    \caption{\textbf{RMSE Curves:} We plot the RMSE values for all methods. The prediction window is 5 seconds corresponding to a frame length of 50 for the NGSIM dataset.}
    \label{fig: rmse}
    \vspace{-5pt}
\end{figure}

\subsection{Behavior Prediction Results} 
We follow the behavior prediction algorithm described in Section~\ref{subsec: behavior_protocol}. The values for $\lambda_1$ and $\lambda_2$ are based on the ground truth labels and are hidden from the test set. We observe a weighted accuracy of 92.96\% on the Lyft dataset, 84.11\% on the Argoverse dataset, and 96.72\% on the Apolloscape dataset. In the case of Lyft, Figures~\ref{fig: lyft_gt} and \ref{fig: lyft_pred} show the ground truth and predictions for Lyft, respectively. We plot the value of $\theta^{'}$ on the vertical axis and the road-agent I.D.s on the horizontal axis. More similarity across the two plots indicates higher accuracy. For instance, the red (aggressive) and blue (conservative) dotted regions in~\ref{fig: lyft_gt} and~\ref{fig: lyft_pred} are nearly identical indicating a greater number of correct classifications. Similar results follow for the Apolloscape and Argoverse datasets, which we show in the supplementary material due to lack of space. Due to the lack of diverse behaviors in the NGSIM dataset, we do not perform behavior prediction on the NGSIM.

An interesting observation is that road-agents towards the end of the x-axis appear late in the traffic video while road-agents at the beginning of the x-axis appear early in the video. The variation in behavior class labels, therefore, decreases towards the end of the x-axis. This intuitively makes sense as $\theta^{'}$ for a road-agent depends on the number of distinct neighbors that it observes. This is difficult for road-agents towards the end of the traffic video.

\subsection{Ablation Study of the Radius Parameter ($\mu$)}
\label{subsec: mu_ablation}
We conducted ablation experiments in which we vary the radius parameter $\mu$ (See Section~\mbox{\ref{subsec: spectral graph theory}} for a discussion on $\mu$) from $0$ to $20$. We obtained results on the Apolloscape, Argoverse, Lyft, and NGSIM datasets which we present in Table\mbox{~\ref{tab: mu_ablation}}. We measured the average RMSE values over 3 range intervals: $\mu = 0$,  $0 < \mu \leq 10$, and $10 < \mu \leq 20$. We use range intervals to clearly and succinctly capture the trend of the RMSE values for $\mu>10$ meters and $\mu<10$ meters. We observe that the best performance is achieved from the latter range ($0 < \mu \leq 10$).
\begin{figure}
\centering
\begin{subfigure}[h]{.8\columnwidth}
    \includegraphics[width=\columnwidth]{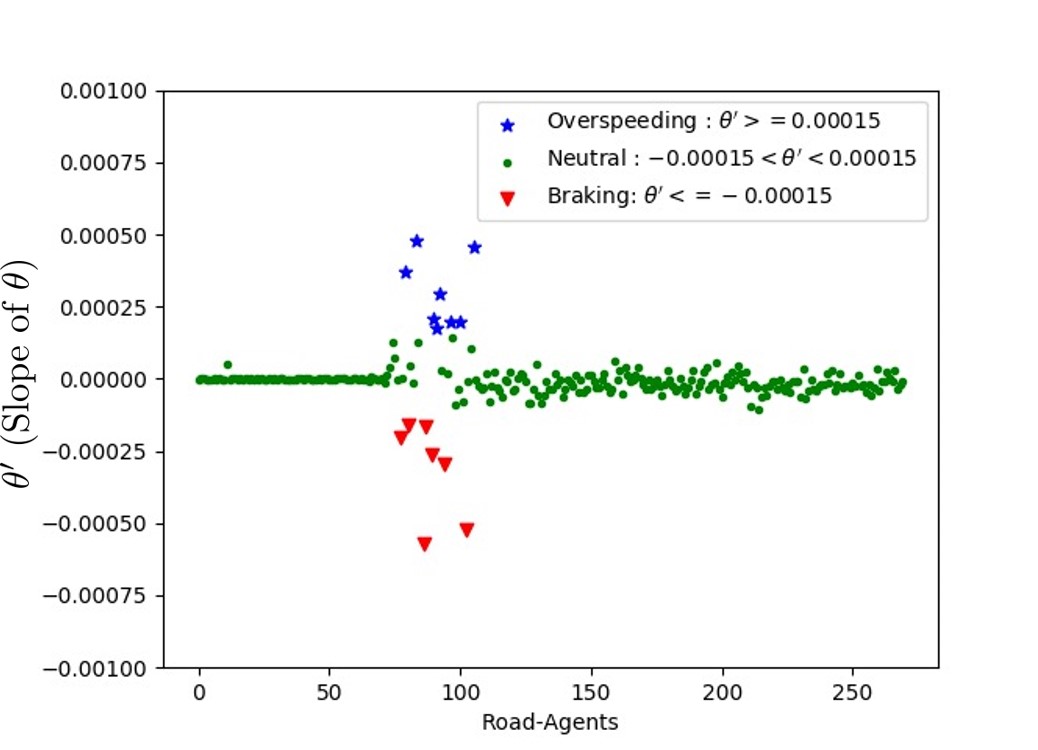}
    \caption{Lyft Ground-Truthwith $\lambda$=$0.00015$.}
    \label{fig: lyft_gt}
  \end{subfigure}
   \begin{subfigure}[h]{.8\columnwidth}
    \includegraphics[width=\columnwidth]{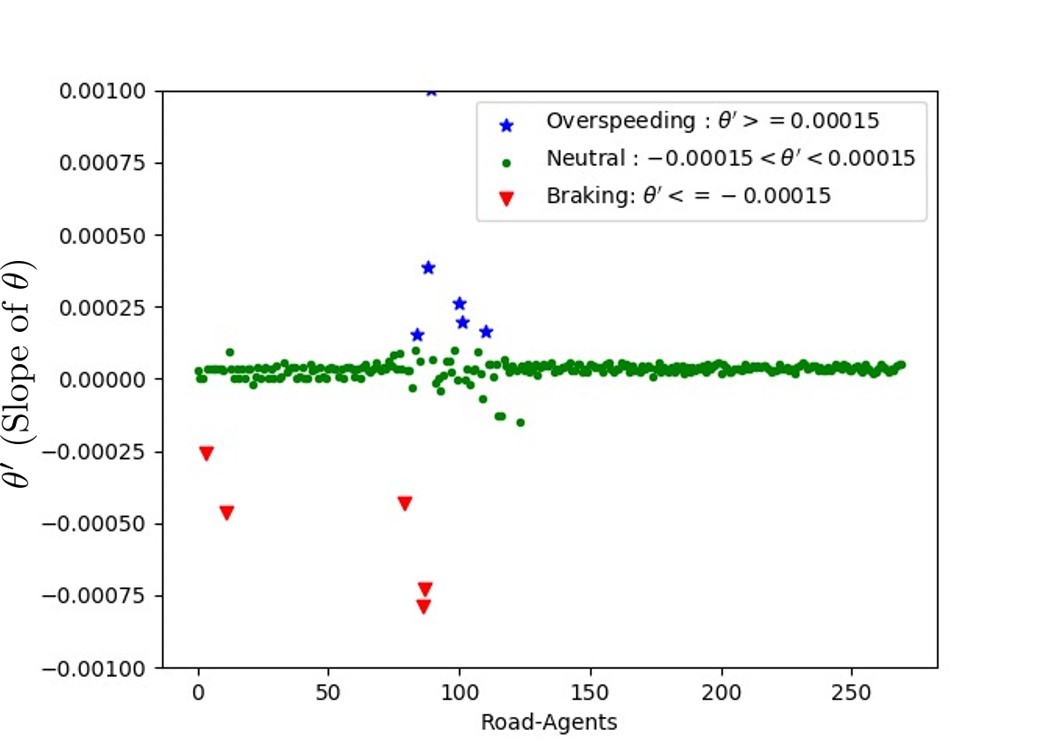}
    \caption{Lyft Behavior Predictions.}
    \label{fig: lyft_pred}
  \end{subfigure}

\caption{\textbf{Behavior Prediction Results:} We classify the three behaviors-- overspeeding(\textcolor{blue}{blue}), neutral(\textcolor{green}{green}), and underspeeding(\textcolor{red}{red}), for all road-agents in the Lyft, Argoverse, and Apolloscape datasets, respectively. The y-axis shows $\theta^{'}$ and the x-axis denotes the road-agents. We follow the behavior prediction protocol described in Section~\ref{subsec: behavior_protocol}. Each figure in the top row represents the ground-truth labels, while the bottom row shows the predicted labels. In our experiments, we set $\lambda = \lambda_1 = - \lambda_2$.}
  \label{fig: behavior}
  \vspace{-8pt}
\end{figure}

        

\begin{table}
  \caption{Ablation experiments of the radius parameter $\mu$. Each column contains averaged RMSE values over the corresponding range interval. \textit{Conclusion:} The optimum results are obtained by setting $0 < \mu \leq 10$ meters.}   
  \centering
  \begin{tabular}{lccc}
  \toprule
    Dataset  & $\mu = 0$ & $0 < \mu \leq 10$ &  $10 < \mu \leq 20$ \\
    \midrule
    Apolloscape & 2.14 & \textbf{1.12} &2.62\\
    Argoverse  &2.40  & \textbf{0.99}&3.15 \\
    Lyft Level 5  &5.77  & \textbf{2.65}&3.36\\
    NGSIM &1.31 &\textbf{0.40}  &2.03\\
    \bottomrule
  \end{tabular}
    \label{tab: mu_ablation}
  \vspace{-12pt}
\end{table}


It is clear that setting $\mu=0$ and thus ignoring neighborhood information in dense traffic severely degrades performance. But on the other hand, increasing the radius beyond $10$ meters also increases the RMSE error. This is because by increasing the radius beyond a certain threshold, we inadvertently include in our spectral clustering algorithm those road-agents that are too far to interact with the ego-agent. In order to accommodate these ``far-away'' road-agents, the clusters expand and shift the cluster center from its true center. This phenomenon is common in statistics where an outlier corrupts the data distribution. The far-away agents are outliers in the spectral clustering algorithm, thereby leading to an increase in RMSE. We conclude that our method produces optimum results for $0 < \mu \leq 10$ in dense traffic systems.
\section{Conclusion, Limitations, and Future Work}

We present a unified algorithm for trajectory prediction and behavior prediction of road-agents. We use a two-stream LSTM network in which the first stream predicts the trajectories, while the second stream predicts the behavior of road-agents. We also present a regularization algorithm to reduce long-term prediction errors.

Our method has some limitations. Currently, we use only one feature to design our behavior prediction model, which may not be able to generalize to new traffic scenarios. In addition, our training is slow and takes several hours due to the number of computations required for computing the traffic-graphs and corresponding Laplacian matrices. We plan to make our behavior prediction model data-driven, rather than rule-based. We will also explore ways to improve trajectory prediction using our behavior prediction algorithm.

\section{Acknowledgements}

This work was supported in part by ARO Grants W911NF1910069 and W911NF1910315, Semiconductor Research Corporation (SRC), and Intel.





{\small
\bibliographystyle{IEEEtran}
\bibliography{refs}
}

\end{document}